\documentclass[runningheads]{llncs}

\usepackage[misc]{ifsym}

\usepackage{amsmath,amssymb}
\usepackage{mathtools}
\usepackage{bm}
\usepackage{graphicx}
\usepackage{booktabs}
\usepackage{multirow}
\usepackage{algorithm}
\usepackage{algorithmic}
\usepackage{url}
\usepackage[hidelinks]{hyperref}
\usepackage{tikz}
\usepackage{makecell}
\usetikzlibrary{arrows.meta,positioning,calc}
\usepackage{wrapfig}
\usepackage{xcolor}

\newcommand{\R}{\mathbb{R}}
\newcommand{\I}{\mathbf{I}}

\newcommand{\normF}[1]{\left\lVert #1 \right\rVert_F}
\newcommand{\normtwo}[1]{\left\lVert #1 \right\rVert_2}

\newif\ifappendixonly
\newif\ifmainonly
\begin{document}

\title{Exact Federated Continual Unlearning for Ridge Heads on Frozen Foundation Models}
\titlerunning{Exact Federated Continual Unlearning for Ridge Heads}

\author{Yijun Quan\inst{1}\Letter \and Wentai Wu\inst{2} \and Giovanni Montana\inst{1}}
\authorrunning{Quan et al.}
\institute{WMG, University of Warwick, Coventry CV4 7AL, UK \and  Department of Computer Science, Jinan University, Guangzhou 510632, China\\ 
\email{\{yijun.quan, g.montana\}@warwick.ac.uk} \\ \email{wentaiwu@jnu.edu.cn}}

\toctitle{Exact Federated Continual Unlearning For Ridge Heads on Frozen Foundation Models}
\tocauthor{Yijun Quan, Wentai Wu, Giovanni Montana}
\ifappendixonly
\else
\maketitle
\begin{abstract}
Foundation models are commonly deployed as frozen feature extractors with a small trainable head to adapt to private, user-generated data in federated settings. The ``right to be forgotten'' requires removing the influence of specific samples or users from the trained model on demand. Existing federated unlearning methods target general deep models and rely on approximate reconstruction or selective retraining, making exactness costly or elusive. We study this problem in a practically relevant but under-explored regime: a frozen foundation model with a ridge-regression head. The exact optimum depends on the data only through two additive sufficient statistics, which we turn into a communication protocol supporting an arbitrary stream of \emph{add} and \emph{delete} requests via fixed-size messages. The server maintains a head that is, in exact arithmetic, \emph{pointwise identical} to centralized retraining after every request. We provide deterministic retrain-equivalence guarantees, order and partition invariance, two server-side variants, and a Bayesian certificate of zero KL divergence. Experiments on four benchmarks confirm the guarantees: both variants match centralized ridge retraining to within $10^{-9}$ relative Frobenius error and complete each request at orders-of-magnitude lower cost than federated retraining baselines.
\end{abstract}


\section{Introduction}

Large pretrained models are increasingly used as \emph{frozen} backbones for
downstream adaptation: the backbone $\phi(\cdot)$ remains fixed, while a lightweight
head is trained using private data from an organization or from end users.
This practice is particularly attractive in privacy-sensitive federated settings where
raw data cannot be centralized: clients compute representations locally and share only
head-related information with a server.
The pattern is now dominant in regulated domains such as healthcare, finance, and
legal services, where foundation models serve as shared feature extractors and
per-institution or per-user heads are trained on sensitive local data.

At the same time, regulatory and operational requirements (e.g., GDPR's right to
erasure, user revocation, data corrections, toxicity removal) demand that deployed
models support \emph{data deletion}. This need has given rise to the field of machine
unlearning~\cite{bourtoule2021machine,cao2015towards}, which studies how to remove
the influence of specific training examples or users from learned models without
naively retraining from scratch. Deleting raw records is insufficient: if records
influenced the learned parameters, their effect must be excised from the model itself.
In centralized learning, the gold standard for \emph{exact} unlearning is to retrain
from scratch on the retained data, yielding parameters \emph{identical} to those
obtained had the deleted data never been used. In federated learning (FL), however,
such exact retraining is prohibitive: it requires coordinating many clients and
repeating many communication rounds. This tension motivates the study of
\emph{exact federated unlearning}, with methods proposed to cluster clients and limit
retraining to a smaller group~\cite{FedCIO}, enforce stochastic stability and
selectively retrain~\cite{FATS}, or engineer quantization-based exactness criteria
that reduce the need for retraining~\cite{exactfun}.
These approaches target general deep models trained by iterative optimization, where
exact parameter removal is genuinely difficult.
\emph{None of them exploit the analytic structure that arises when the backbone is
frozen}, leaving exact federated unlearning in this increasingly prevalent regime
unstudied.

This paper fills that gap. We focus on a deployment pattern that is uniquely amenable
to \emph{analytic} updates: a \emph{frozen} deterministic feature extractor and a
\emph{ridge} (least-squares with $\ell_2$ regularization) linear head. Ridge heads
are widely used either directly for classification and regression, or as a tractable
surrogate for fast adaptation; for instance, a ridge head on frozen ViT features
already yields strong few-shot image classification
performance~\cite{bar2024frozen}. Crucially, in this regime the exact optimum is
available in closed form and depends on the data only through two second-order
sufficient statistics: the feature Gram matrix and the feature--label moment.
As a consequence, \emph{federated continual unlearning}---supporting a stream of both
additions and deletions across clients---reduces to maintaining these statistics
exactly and updating them additively or subtractively. The result is a protocol that
is simultaneously exact, communication-efficient, and trivially supports continual
add/delete streams without any retraining.

\paragraph{Contributions.}
We make the following contributions:
\begin{itemize}
    \item We formalize \emph{federated continual unlearning} for frozen foundation
    models with a ridge head, covering both sample-level and client-level deletions,
    and state a deterministic retrain-equivalence goal.

    \item We derive a federated protocol in which clients transmit only fixed-size
    sufficient-statistic messages for every add/delete event; the server maintains
    global retained-set statistics and recovers the head \emph{exactly equal} to
    centralized retraining after each event, in a single communication round.

    \item We give two server-side implementations: a numerically robust \emph{exact
    solver} (Variant~A) that maintains statistics and solves an SPD system each round,
    and an \emph{incremental inverse tracker} (Variant~B) that applies
    Sherman--Morrison--Woodbury (SMW) updates, with feasibility conditions for exact
    downdates and practical reset rules.

    \item We prove deterministic exactness, order invariance (across clients, samples,
    and event interleavings), client-partition invariance, and equivalence of the two
    variants in exact arithmetic.

    \item Interpreting ridge regression as Bayesian linear regression, we show that
    the protocol maintains the \emph{exact posterior} after any deletion stream,
    implying zero KL divergence between the unlearned and retrained posteriors.

    \item We validate both variants experimentally on four benchmarks (CIFAR-10,
    CIFAR-100, FeMNIST, Sentiment140): both match centralized ridge retraining to
    within $10^{-9}$ relative Frobenius error (fp64) across varying client counts and
    non-IID partitions, and handle a stream of 200 single-point deletion requests at
    orders-of-magnitude lower wall-clock cost than FedAvg-based exact federated
    unlearning baselines.
\end{itemize}

Our results provide exact federated continual unlearning for the ridge head given
frozen features. They do \emph{not} address information stored inside a backbone
trained end-to-end on revocable data; backbone unlearning is orthogonal and remains
challenging in general.

\section{Related work}
We discuss three threads most relevant to this work: general machine unlearning,
exact federated unlearning, and analytic learning with frozen backbones.

\paragraph{Machine unlearning.}
The term \emph{machine unlearning} was introduced in early work on making systems
forget~\cite{cao2015towards}. The literature is commonly grouped into
\emph{approximate} and \emph{exact} unlearning.
Approximate unlearning, often coupled with certified removal, is studied via
stability and influence-style arguments
(e.g.,~\cite{guo2020certified,neel2021descent,izzo2021approximate}), and provides
probabilistic or bounded guarantees rather than strict equivalence to full retraining.
Exact unlearning aims to make the updated model indistinguishable from one retrained
from scratch on the modified dataset. SISA training~\cite{bourtoule2021machine} is
the canonical example: it partitions data into shards and retrains only affected
sub-models upon deletion, achieving retrain-equivalent behavior at the cost of
additional storage and some utility loss~\cite{koch2023no}.
Further works improve the accuracy--efficiency tradeoff within sharding-based
frameworks~\cite{yan2022arcane,dukler2023safe}, exploit analytical closed-form
updates~\cite{guo2020certified,tang2025acu,yang2025muso}, or tailor exact deletion
guarantees to knowledge distillation~\cite{quanefficient} and federated learning.
Our work belongs to the exact unlearning category and exploits closed-form structure,
but targets the federated continual setting with both additions and deletions.

\paragraph{Exact federated unlearning.}
Exact federated unlearning inherits the challenges of both FL and exact unlearning:
non-IID data, limited server access, and communication bottlenecks, together with the
requirement that every update be equivalent to retraining on the retained data.
Since exact unlearning typically requires partial network retraining, the literature
focuses on reducing how often and how much retraining is needed.
FedCIO~\cite{FedCIO} clusters clients by importance and restricts retraining to a
selected subset, localizing the unlearning cost.
Exact-Fun~\cite{exactfun} reduces retraining probability via quantization-based
exactness criteria.
FATS~\cite{FATS} designs TV-stable FL algorithms so that retraining is needed only
with bounded probability over a stream of requests.
All three substantially reduce retraining frequency, but partial or full retraining
is still occasionally triggered and remains costly in federated settings.
By contrast, our protocol requires \emph{no retraining at all}: each unlearning
request is handled in a single communication round via a closed-form update,
which is possible precisely because the backbone is frozen.

\paragraph{Analytic learning with frozen backbones.}
A line of work in continual learning studies analytic heads on frozen features to
avoid gradient instability and achieve exact retrain equivalence in continual
\emph{learning} settings~\cite{zhuang2022acil,tang2025afcl}.
ACU~\cite{tang2025acu} extends this to continual \emph{unlearning} in a
\emph{centralized} setting, deriving exact ridge downdates via analytic inverse
tracking.
We build on the same algebraic foundation but depart in two key respects: we operate
in a \emph{federated} setting with $K$ clients, and we support a continual
\emph{add+delete} stream rather than deletions alone.
The analytic structure becomes a \emph{communication protocol}: clients transmit
fixed-size sufficient-statistic messages, and the server maintains exact global
retrain equivalence without ever accessing raw features.

\section{Setup and problem statement}

We consider a two-stage model: a frozen feature extractor $\phi:\mathcal{X}\to\R^d$
and a trainable linear head $W\in\R^{d\times c}$.
For a sample $(x_i,y_i)$, let $f_i \coloneqq \phi(x_i)\in\R^{d}$ denote its frozen
feature vector and $y_i\in\R^{c}$ its label (one-hot or soft for classification,
real-valued for regression).
For a dataset of $n$ samples, we collect features and labels into matrices
$F\in\R^{n\times d}$ and $Y\in\R^{n\times c}$, so the model predicts $\hat{Y}=FW$.
We assume every client applies the same frozen $\phi$ with identical preprocessing,
and that $\phi$ is deterministic at inference time (evaluation mode, no dropout);
if stochasticity is unavoidable, exactness is recovered by caching $f_i$ at add time
and reusing it at deletion time. An intercept can be included by augmenting
$f_i\leftarrow[f_i^\top\;1]^\top$.

We consider a server and $K$ clients. At time $t$, client $k$ holds a local retained
multiset $\mathcal{D}_{k,t}$. Between rounds $t-1$ and $t$, each client may
\emph{add} a batch $\mathcal{D}^+_{k,t}$ of new samples and/or \emph{delete} a batch
$\mathcal{D}^-_{k,t}\subseteq\mathcal{D}_{k,t-1}$ of previously retained samples.
Raw inputs $x$ are never sent to the server. We write
$\mathcal{D}_t \coloneqq \biguplus_{k=1}^K \mathcal{D}_{k,t}$ for the global
retained multiset and $(F_t,Y_t)$ for the corresponding feature and label matrices.

The objective is to maintain a sequence of heads $\{W_t\}_{t\ge 0}$, shared by all
clients and the server, such that for every $t$,
\begin{equation}
\label{eq:goal}
W_t = W^\star(F_t,Y_t),
\end{equation}
where $W^\star(F,Y)$ is the unique ridge-regression minimizer on $(F,Y)$.
In machine unlearning, ``exact'' typically means the unlearned model has the same
distribution as one retrained from scratch on the retained data. Since ridge
regression with $\gamma>0$ has a unique deterministic solution, distributional
equality reduces to pointwise equality of parameters, which is precisely
Equation~\eqref{eq:goal}.

\section{Ridge heads and sufficient statistics}

We use the matrix ridge objective
\begin{equation}
\label{eq:ridge}
\min_{W\in\R^{d\times c}} \;\; \normF{Y-FW}^2 + \gamma \normF{W}^2,\qquad \gamma>0.
\end{equation}
Setting the gradient to zero yields the closed-form minimizer
\begin{equation}
\label{eq:closedform}
W^\star = (S+\gamma \I)^{-1}G,
\end{equation}
where $S \coloneqq F^\top F \in \R^{d\times d}$ and $G \coloneqq F^\top Y \in
\R^{d\times c}$ are the sufficient statistics of the data.
\begin{equation}
\label{eq:stats}
S \coloneqq F^\top F,\qquad G \coloneqq F^\top Y.
\end{equation}
Training thus depends on the data only through the pair $(S,G)$.

A key property is that sufficient statistics are additive across disjoint subsets: if
a dataset decomposes into parts with statistics $(S_j,G_j)$, then the union has
$S=\sum_j S_j$ and $G=\sum_j G_j$, and deleting a subset simply subtracts its
contribution. This additivity is the algebraic foundation for the protocol in
Section~\ref{sec:protocol}. It also explains why second-order summaries are necessary:
the ridge solution depends on $S=\sum_i f_if_i^\top$, not merely on the first-order
mean $\sum_i f_i$, so no protocol relying only on first-order information can achieve
deterministic exactness under arbitrary adds and deletes
(Lemma~\ref{lem:need_second_order}, Section~\ref{sec:2nd_need}).

\section{Exact federated continual unlearning protocol}
\label{sec:protocol}

We treat $t=0$ as an empty initial state ($S_0=0$, $G_0=0$) and $t=1$ as the round
in which all initial training data is added. At each subsequent round $t\geq1$,
client $k$ forms feature and label matrices $(F^+_{k,t},Y^+_{k,t})$ for additions
and $(F^-_{k,t},Y^-_{k,t})$ for deletions, and transmits only the corresponding
sufficient statistics:
\[
S^{\pm}_{k,t}=(F^{\pm}_{k,t})^\top F^{\pm}_{k,t},\qquad
G^{\pm}_{k,t}=(F^{\pm}_{k,t})^\top Y^{\pm}_{k,t}.
\]
To form $(S^-_{k,t},G^-_{k,t})$ exactly, the client must recover the features of
deleted samples, either by recomputing $f=\phi(x)$ under the deterministic frozen
backbone, or by caching $f$ at add time and discarding $x$ later.
The server aggregates client messages:
\begin{equation}
\label{eq:agg}
S_t^{\pm} = \sum_{k=1}^K S^{\pm}_{k,t},\qquad
G_t^{\pm} = \sum_{k=1}^K G^{\pm}_{k,t},
\end{equation}
and maintains a retained-statistics ledger:
\begin{equation}
\label{eq:ledger}
S_t \leftarrow S_{t-1} + S_t^+ - S_t^-,
\qquad
G_t \leftarrow G_{t-1} + G_t^+ - G_t^-.
\end{equation}
The ridge head is then recovered as $W_t = H_t^{-1}G_t$ where
$H_t \coloneqq S_t + \gamma\I$.
The two variants below differ only in how the server computes this update.

\paragraph{Variant A: exact recomputation via SPD solves.}
Variant~A forms $H_t$ explicitly, computes its Cholesky factorization
$H_t=L_tL_t^\top$, and recovers $W_t$ via two triangular solves.
This avoids forming $H_t^{-1}$ explicitly and is numerically stable; it serves as
the recommended baseline.

\paragraph{Variant B: incremental inverse updates via SMW.}
When the per-round Gram changes admit a low-rank factorization, the server can update
the inverse analytically. Let $T=H^{-1}$ and write $\Delta S=U^\top U$ for
$U\in\R^{r\times d}$.

\paragraph{Add.}
With $H^+=H+\Delta S$,
\begin{equation}
\label{eq:smw_add}
T^{+} = T - T U^\top (\I_r + U T U^\top)^{-1} U T,
\end{equation}
and the updated head is
\begin{equation}
\label{eq:w_add}
W^+ = W - T^+U^\top(UW) + T^+\Delta G.
\end{equation}

\paragraph{Delete.}
With $H^-=H-\Delta S$, assuming $H^-\succ 0$,
\begin{equation}
\label{eq:smw_del}
T^{-} = T + T U^\top (\I_r - U T U^\top)^{-1} U T,
\end{equation}
and
\begin{equation}
\label{eq:w_del}
W^- = W + T^-U^\top(UW) - T^-\Delta G.
\end{equation}

\paragraph{Constructing $U$ via client QR factors.}
In practice, each client applies a thin QR factorization $F_{k,t}=Q_kR_k$, so that
$S_{k,t}=R_k^\top R_k$, and transmits only the upper-triangular factor $R_k$.
The server stacks these into $U=[R_1^\top,\dots,R_K^\top]^\top\in\R^{Kr\times d}$,
giving $U^\top U=\sum_k R_k^\top R_k=S_t^{\pm}$.
This avoids transmitting raw feature matrices and sidesteps ill-conditioned Cholesky
factorizations on $\Delta S$, which can occur when add/delete batches are small.

\paragraph{Feasibility and resets.}
For deletions the matrix $\I_r-UTU^\top$ must be SPD.
In finite precision, repeated downdates can accumulate drift; we recommend falling
back to Variant~A whenever the Cholesky factorization of $\I_r-UTU^\top$ fails or
becomes ill-conditioned.

The Algorithms~\ref{alg:variantA}--\ref{alg:variantB} summarize the full client--server protocol; Figure~\ref{fig:overview} gives a high-level overview. Algorithm~\ref{alg:variantA} specifies the combined client–server routine for Variant A, which aggregates sufficient statistics and recomputes the ridge head via SPD solves, while Algorithm~\ref{alg:variantB} gives the corresponding routine for Variant B, which instead applies Sherman–Morrison–Woodbury updates; both variants produce a head that is exact to centralized ridge retraining after every request.

\begin{algorithm}
\caption{Variant A at round $t$ (sufficient statistics and exact recomputation)}
\label{alg:variantA}
\begin{algorithmic}[1]
\STATE \textbf{Client $k$ inputs:} frozen feature extractor $\phi$,
       add batch $\mathcal{D}^+_{k,t}$, delete batch $\mathcal{D}^-_{k,t}$
\STATE Compute features for additions:
       $F^+_{k,t}\leftarrow [\phi(x)]_{(x,y)\in \mathcal{D}^+_{k,t}}$,
       $Y^+_{k,t}\leftarrow [y]$
\STATE Compute features for deletions:
       $F^-_{k,t}\leftarrow [\phi(x)]_{(x,y)\in \mathcal{D}^-_{k,t}}$ (or cached),
       $Y^-_{k,t}\leftarrow [y]$
\STATE Form sufficient statistics:
\STATE \hspace{0.8em}$S^+_{k,t}=(F^+_{k,t})^\top F^+_{k,t}$,\;\;
                     $G^+_{k,t}=(F^+_{k,t})^\top Y^+_{k,t}$
\STATE \hspace{0.8em}$S^-_{k,t}=(F^-_{k,t})^\top F^-_{k,t}$,\;\;
                     $G^-_{k,t}=(F^-_{k,t})^\top Y^-_{k,t}$
\STATE Client sends $(S^+_{k,t},G^+_{k,t},S^-_{k,t},G^-_{k,t})$ to server
       (optionally under secure aggregation).

\STATE \textbf{Server inputs:} previous $(S_{t-1},G_{t-1})$, regularizer $\gamma$
\STATE Aggregate client messages to obtain $(S^+_t,G^+_t,S^-_t,G^-_t)$ as in \eqref{eq:agg}
\STATE Update retained-statistics ledger:
       $S_t\leftarrow S_{t-1}+S^+_t-S^-_t$,\;
       $G_t\leftarrow G_{t-1}+G^+_t-G^-_t$
\STATE Form $H_t\leftarrow S_t+\gamma I$; compute Cholesky $H_t=L_tL_t^\top$
\STATE Solve $H_t W_t = G_t$ via triangular solves (no explicit inverse)
\STATE Broadcast $W_t$ for inference
\end{algorithmic}
\end{algorithm}

\begin{algorithm}
\caption{Variant B at round $t$ (SMW updates using client $R$ factors)}
\label{alg:variantB}
\begin{algorithmic}[1]
\STATE \textbf{Client $k$ inputs:} frozen feature extractor $\phi$, add batch
       $\mathcal{D}^+_{k,t}$, delete batch $\mathcal{D}^-_{k,t}$
\STATE Compute features for additions:
       $F^+_{k,t}\leftarrow [\phi(x)]_{(x,y)\in \mathcal{D}^+_{k,t}}$, \;
       $Y^+_{k,t}\leftarrow [y]$
\STATE Compute features for deletions:
       $F^-_{k,t}\leftarrow [\phi(x)]_{(x,y)\in \mathcal{D}^-_{k,t}}$ (or cached), \;
       $Y^-_{k,t}\leftarrow [y]$
\STATE Compute thin QR factorizations:
       $F^+_{k,t} = Q^+_{k,t} R^+_{k,t}$, \;
       $F^-_{k,t} = Q^-_{k,t} R^-_{k,t}$
\STATE Form local first-order summaries:
       $G^+_{k,t} = (F^+_{k,t})^\top Y^+_{k,t}$, \;
       $G^-_{k,t} = (F^-_{k,t})^\top Y^-_{k,t}$
\STATE Client sends $(R^+_{k,t},R^-_{k,t},G^+_{k,t},G^-_{k,t})$ to server.

\STATE \textbf{Server inputs:} previous $G_{t-1}$,
       inverse $T_{t-1} = (S_{t-1}+\gamma I)^{-1}$, current $W_{t-1}$
\STATE Aggregate client messages to obtain $R^+_t,R^-_t,G^+_t,G^-_t$
\STATE $\Delta G_t \leftarrow G^+_t - G^-_t$
\STATE Update retained-statistics ledger:
       $G_t \leftarrow G_{t-1} + \Delta G_t$
\STATE Let $U^+_t \leftarrow R_t^+$ and $U^-_t \leftarrow R_t^-$.
       Apply a Sherman--Morrison--Woodbury update for the additions.
       Note the weight formula uses $G_t^+$ (the additions-only label moment) together
       with the previous head $W_{t-1}$:
       \[
       \begin{aligned}
       T_t^+ &=
             T_{t-1} - T_{t-1} {U_t^{+}}^\top (I_r + U_t^+ T_{t-1} {U_t^+}^\top)^{-1}
             U_t^+ T_{t-1},\\
        W_t^+ &= (I-T_t^+{U_t^+}^\top U_t^+)W_{t-1}+T_t^+G_t^+,
       \end{aligned}
       \]
\STATE Then apply a second SMW update for deletions with $G_t^-$ (the deletions-only labbel moment) to obtain the final inverse and
       head:
       \[\begin{aligned}
       T_t &= T_t^+ + T_t^+ {U_t^-}^\top (I_r - U_t^- T_t^+ {U_t^-}^\top)^{-1}
              U_t^- T_t^+ \\
       W_t &= (I+T_t{U_t^-}^\top U_t^-)W_t^+ -T_tG_t^-,
       \end{aligned}\]
\STATE Broadcast $W_t$ for inference
\end{algorithmic}
\end{algorithm}

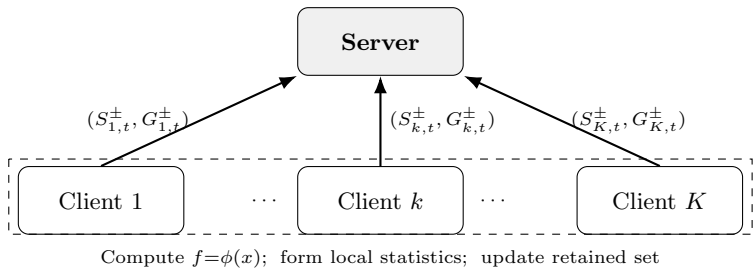
\begin{figure}[t]
\centering
\begin{tikzpicture}[node distance=1.2cm,>=Latex]

\node[draw, rounded corners, minimum width=2.2cm, minimum height=0.9cm,
      fill=gray!12] (srv) {\textbf{Server}};

\node[draw, rounded corners, minimum width=2.2cm, minimum height=0.9cm,
      below left=1.2cm and 1.5cm of srv] (c1) {Client $1$};
\node[draw, rounded corners, minimum width=2.2cm, minimum height=0.9cm,
      below=1.2cm of srv] (c2) {Client $k$};
\node[draw, rounded corners, minimum width=2.2cm, minimum height=0.9cm,
      below right=1.2cm and 1.5cm of srv] (c3) {Client $K$};

\node[left=0.1cm of c2]  {\scriptsize $\cdots$};
\node[right=0.1cm of c2] {\scriptsize $\cdots$};

\draw[->, thick] (c1.north) -- node[left]  {\scriptsize $(S^{\pm}_{1,t},G^{\pm}_{1,t})$} (srv.south west);
\draw[->, thick] (c2.north) -- node[right] {\scriptsize $(S^{\pm}_{k,t},G^{\pm}_{k,t})$} (srv.south);
\draw[->, thick] (c3.north) -- node[right] {\scriptsize $(S^{\pm}_{K,t},G^{\pm}_{K,t})$} (srv.south east);

\draw[dashed] ($(c1.west)+(-0.12,0.55)$) rectangle ($(c3.east)+(0.12,-0.45)$);
\node[align=center, below=0.08cm of c2] {\scriptsize
  Compute $f{=}\phi(x)$;\; form local statistics;\; update retained set};

\end{tikzpicture}
\caption{Protocol overview (Variant A). Clients compute frozen features locally
and transmit only fixed-size sufficient-statistic messages $(S^\pm_{k,t}, G^\pm_{k,t})$.
The server maintains a retained-statistics ledger
($S_t \leftarrow S_{t-1}{+}S_t^{+}{-}S_t^{-}$, $G_t \leftarrow G_{t-1}{+}G_t^{+}{-}G_t^{-}$)
and recovers the exact ridge head by solving $(S_t{+}\gamma I)W_t = G_t$.}
\label{fig:overview}
\end{figure}


\section{Theoretical properties}
\label{sec:theory}

We establish five properties of the protocol: (i) the server head is at all times
pointwise identical to centralized retraining on the retained data; (ii) the result
is invariant to the order of updates and the assignment of samples to clients;
(iii) exact downdates are feasible whenever the retained Gram matrix remains positive
definite after deletion; (iv) the protocol maintains the exact Bayesian posterior,
yielding a zero-KL unlearning certificate; and (v) communication and computation
scale with the rank of the per-round update rather than the size of the full dataset.
Throughout we assume $\gamma>0$ and that each deletion request is valid (the deleted
multiset is a subset of the retained set), so $S_t\succeq 0$ and
$H_t=S_t+\gamma\I\succ 0$ for all $t$.

\subsection{Deterministic retrain equivalence}

\begin{theorem}[Deterministic retrain equivalence]
\label{thm:det_exact}
Let $(S_t,G_t)$ be the exact sufficient statistics of the global retained dataset
$\mathcal{D}_t$. Then the server head $W_t=(S_t+\gamma\I)^{-1}G_t$ equals the
unique ridge-regression solution obtained by centralized retraining from scratch
on $\mathcal{D}_t$.
\end{theorem}

\begin{proof}
Centralized ridge regression depends on the data only through $(S_t,G_t)$ and has
unique optimum~\eqref{eq:closedform}. The protocol maintains $(S_t,G_t)$ exactly
via the ledger update~\eqref{eq:ledger}, hence computes the same $W_t$.
\end{proof}

In machine unlearning terms, since both retraining and the protocol produce the same
deterministic $W_t$, the output distributions coincide pointwise: the unlearned model
is indistinguishable from one retrained from scratch on $\mathcal{D}_t$.

\subsection{Order and client-partition invariance}

\begin{theorem}[Order and partition invariance]
\label{thm:invariance}
Fix a final retained multiset $\mathcal{D}_t$. The statistics $(S_t,G_t)$ and hence
$W_t$ are invariant to (i) the order in which clients report updates, (ii) any
reassignment of samples across clients, and (iii) any interleaving of add/delete
operations that yields the same final retained multiset.
\end{theorem}

\begin{proof}
All three transformations preserve the final retained multiset. Since $(S_t,G_t)$
are sums over that multiset and addition is associative and commutative,
$(S_t,G_t)$ are unchanged, and so is $W_t=(S_t+\gamma\I)^{-1}G_t$.
\end{proof}

\subsection{Variant A vs Variant B}
\label{sec:variant_equiv}
\begin{theorem}[Equivalence of Variant A and Variant B in exact arithmetic]
\label{thm:equiv_AB}
Within each round $t$, suppose the aggregated round statistics admit exact
factorizations $S_t^+ = (U_t^+)^\top U_t^+$ and $S_t^- = (U_t^-)^\top U_t^-$.
Variant~B applies two sequential SMW steps per round: first an add step using $U_t^+$,
then a delete step using $U_t^-$, with the intermediate feasibility condition
$S_{t-1} + \gamma I + S_t^+ - S_t^- \succ 0$ ensuring the downdate is valid.
If Variant~B starts from the same initial $T_0=(S_0+\gamma I)^{-1}$ as
Variant~A and maintains $(S_t,G_t)$ exactly via~\eqref{eq:ledger},
then in exact arithmetic it holds for all $t$ that
\[
T_t^{B}=(S_t+\gamma I)^{-1} = T_t^{A}, \qquad W_t^{B}=T_t^{B}G_t = W_t^{A}.
\]
\end{theorem}
\begin{proof}
By induction on $t$. After the add step, Variant~B holds
$\widetilde{T}_t = (S_{t-1} + \gamma I + S_t^+)^{-1}$ by the SMW identity.
After the subsequent delete step, it holds
\[
T_t = (\widetilde{H}_t - S_t^-)^{-1} = (S_{t-1} + \gamma I + S_t^+ - S_t^-)^{-1}
= (S_t + \gamma I)^{-1},
\]
again by the SMW identity. Since $G_t$ is maintained exactly
via~\eqref{eq:ledger}, it follows that $W_t^B = T_t^B G_t = W_t^A$.
\end{proof}

\subsection{Feasibility of exact downdates}

\begin{lemma}[Downdate feasibility]
\label{lem:downdate_feasible}
Let $H\succ 0$ and $\Delta S=U^\top U\succeq 0$. The following are equivalent:
(i) $H-\Delta S\succ 0$,\;
(ii) $\I-UH^{-1}U^\top\succ 0$,\;
(iii) $\lambda_{\max}(UH^{-1}U^\top)<1$.
\end{lemma}

\begin{proof}
Write $H-\Delta S=H^{1/2}(\I-H^{-1/2}U^\top UH^{-1/2})H^{1/2}$ and apply the fact
that $A^\top A$ and $AA^\top$ share the same nonzero eigenvalues with
$A=UH^{-1/2}$.
\end{proof}

\subsection{Necessity of second-order information}
\label{sec:2nd_need}
\begin{lemma}[Second-order information is necessary for deterministic exactness]
\label{lem:need_second_order}
Any deterministic protocol that attempts to update $T=(S+\gamma I)^{-1}$ after an add event, but receives only first-order summaries
(e.g., $\sum_i f_i$ and $\sum_i y_i$ and label-weighted first moments), cannot be exact for all possible batches.
\end{lemma}

\begin{proof}
In $\R^2$ with scalar labels $y_i\equiv 1$, consider batch A with $f_1=e_1,f_2=e_2$ and batch B with $f'_1=(1,1)^\top,f'_2=(0,0)^\top$.
Both have the same first-order sum $(1,1)^\top$, but their Gram matrices differ: $\Delta S_A=I$ and $\Delta S_B=\begin{psmallmatrix}1&1\\1&1\end{psmallmatrix}$.
Since the exact update depends on $\Delta S$, a protocol that cannot distinguish these cases cannot be exact on both.
\end{proof}

\subsection{Bayesian posterior certificate}
\label{sec:bayes}

The connection between ridge regression and Bayesian linear regression is classical:
ridge is the MAP estimator under a Gaussian likelihood and isotropic Gaussian prior.
We recall this connection here because it yields a stronger guarantee than
Theorem~\ref{thm:det_exact} at no additional cost: since the protocol maintains
$(S_t,G_t)$ exactly, it preserves not just the posterior mode but the \emph{entire
posterior distribution}. This directly implies a zero-KL unlearning certificate in
the sense of distributional exact unlearning, which is the standard notion in the
unlearning literature~\cite{cao2015towards,bourtoule2021machine}.

Concretely, assume
\[
y_i \mid f_i,W \sim \mathcal{N}(W^\top f_i,\sigma^2 I_c),\qquad
\mathrm{vec}(W)\sim \mathcal{N}(0,\tau^2 I_{dc}),
\]
independently across $i$, and set $\gamma=\sigma^2/\tau^2$.

\begin{lemma}[Posterior in matrix-normal form]
\label{lem:blr}
Let $S=F^\top F$ and $G=F^\top Y$. Then
\[
W\mid(F,Y) \sim \mathcal{MN}(M,\Sigma,I_c),
\quad
\Sigma=\sigma^2(S+\gamma I)^{-1},
\quad
M=(S+\gamma I)^{-1}G.
\]
\end{lemma}

\begin{proof}
Standard Bayesian linear regression calculation; the matrix-normal form follows from
the Kronecker structure of the prior and likelihood.
\end{proof}

\begin{theorem}[Exact Bayesian federated continual unlearning]
\label{thm:bayes_exact}
Let $\Pi_t$ denote the Bayesian posterior given retained data $\mathcal{D}_t$.
The protocol maintains $(S_t,G_t)$ exactly and therefore maintains the exact posterior
\[
\Pi_t = \mathcal{MN}\!\left(W_t,\,\sigma^2(S_t+\gamma I)^{-1},\,I_c\right),
\]
which equals the posterior obtained by centralized recomputation from scratch on
$\mathcal{D}_t$. Consequently,
$\mathrm{KL}(\Pi_t\;\|\;\Pi^{\mathrm{retrain}}_t)=0$.
\end{theorem}

\begin{proof}
By Lemma~\ref{lem:blr}, the posterior depends on the data only through $(S_t,G_t)$.
Since the protocol maintains these statistics exactly, it produces the same posterior
parameters $(M,\Sigma)$ as centralized recomputation, and identical distributions
imply zero KL divergence.
\end{proof}

The tracked inverse $T_t=(S_t+\gamma I)^{-1}$ is, up to $\sigma^2$, the posterior
covariance matrix, giving a natural monotonicity property: valid deletions can only
increase posterior uncertainty in the PSD order, while additions decrease it.

\subsection{Communication and computation costs}

\begin{table}[t]
\centering
\caption{Per-round cost summary in the typical regime $d\gg r$.}
\label{tab:costs}
\begin{tabular}{@{}lcc@{}}
\toprule
 & Communication per round & Server computation \\ \midrule
Variant A & $\Theta(d^2+dc)$ & $\mathcal{O}(d^3+d^2c)$ \\
Variant B & $\Theta(dr+dc)$ & $\mathcal{O}(rd^2+r^3)$ \\
\bottomrule
\end{tabular}
\end{table}

Table~\ref{tab:costs} summarizes per-round costs. Variant~A transmits full $d\times d$ Gram updates, so communication is $\Theta(d^2+dc)$ per client and server computation is dominated by the Cholesky solve at $\mathcal{O}(d^3)$. Variant~B transmits only the $r\times d$ QR factor $R_k$ and the $d\times c$ moment $G_{k,t}$, reducing communication to $\Theta(dr+dc)$; server computation is dominated by the SMW update at $\mathcal{O}(rd^2+r^3)$. When $r\ll d$, Variant~B offers substantial savings in both communication and computation.

\section{Experiments}
We compare the performance of our two ridge-head variants against the exact federated unlearning methods FATS \cite{FATS} and Exact-Fun \cite{exactfun}, using both central ridge retraining and retrain-from-scratch FedAvg \cite{mcmahan2017communication} as baselines. For FATS, Exact-Fun, and FedAvg retraining, we employ a frozen backbone with a single-layer linear head to ensure comparability with our ridge-head setup, and in each communication round the sampled clients locally optimize their head using a cross-entropy loss with standard backpropagation. We use the CIFAR-10, CIFAR-100 \cite{krizhevsky2009learning}, FeMNIST \cite{caldas2018leaf} and Sentiment140 \cite{go2009twitter} datasets for performance evaluation. For CIFAR-10, CIFAR-100 and Sentiment140, we construct non-IID client datasets using a Dirichlet distribution partitioning with $\alpha=0.5$. For FeMNIST, we form clients by subsampling the 3556 writer IDs in the dataset and assigning all samples from each selected writer to a single client. For the three image tasks, we adopt a pretrained DINOv2-ViT B/14 backbone \cite{oquab2023dinov2} with a feature embedding dimension of 768. For the sentiment analysis experiment, we use a RoBERTa backbone pretrained on Twitter data from the TweetEval benchmark \cite{barbieri2020tweeteval}, also with 768-dimensional feature embeddings. All federated learning simulations are carried out using the Flower federated learning framework \cite{beutel2020flower}. Experiments run simulations on a single server with 4 NVIDIA RTX A6000 GPUs. The details of the datasets and the backbone models used for experiments are summarized in Table \ref{tab:exp_setup}.

\begin{table}[t]
\centering
\caption{Experimental setup: datasets, tasks, frozen backbones, and feature
dimensions used in all evaluations.}
\label{tab:exp_setup}
\begin{tabular}{@{}llll@{}}
\toprule
Dataset & Task & Frozen backbone $\phi$ & Feature dim. $d$ \\ \midrule
CIFAR-10  & Image Classification & \makecell[l]{DINOv2} & 768 \\
CIFAR-100  & Image Classification & \makecell[l]{DINOv2} & 768 \\
FeMNIST & Character Classification & \makecell[l]{DINOv2} & 768  \\
Sentiment140 & Sentiment Analysis & RoBERTa & 768 \\
\bottomrule
\end{tabular}
\end{table}

\subsection{Baseline Performance}
We first report the performance of the server-side model trained using the proposed
methods in Table \ref{tab:server_accuracy_variants} to demonstrate that they are
effective while preserving utility and overall predictive performance. We set up the
experiments by running federated learning with 100 clients. For our two variants and
central ridge retrain, they are run as single-round procedures: all 100 clients
participate in that round, and all local samples are used to form the ridge
statistics/solution. Federated Retrain, Exact-Fun, and FATS all use iterative FedAvg
over multiple communication rounds, with FATS additionally employing a more selective
client and data sampling scheme to obtain a TV-stable procedure; the detailed
hyperparameters for all methods and datasets are given in Appendix and we choose them to closely follow the settings used in the corresponding prior work.

\begin{table}[t]
  \centering
  \caption{Server model test accuracy on the evaluated datasets and tasks}
  \label{tab:server_accuracy_variants}
  \begin{tabular}{lcccccc}
    \toprule
    Dataset &
    \makecell{Central ridge\\ retrain} &
    \makecell{FedAvg \\retrain} &
    FATS &
    Exact-Fun &
    Variant A &
    Variant B \\
    \midrule
    CIFAR-10    & 0.9802 & 0.9740 & 0.9634 & 0.9664 & 0.9801 & 0.9802 \\
    CIFAR-100   & 0.8640 & 0.8570 & 0.8623 & 0.8471 & 0.8640 & 0.8640 \\
    FeMNIST     & 0.8084 & 0.7624 & 0.7217  & 0.7562 & 0.8080 & 0.8080 \\
    Sentiment 140 &  0.8291 & 0.8261 & 0.8264 & 0.8190 & 0.8291 & 0.8291\\
    \bottomrule
  \end{tabular}
\end{table}

Overall, all methods achieve strong performance across the evaluated benchmarks, confirming that a simple linear probe on top of a powerful frozen backbone
such as DINOv2 can be sufficient for these tasks.
The primary objective-matched comparison is between the two proposed variants and
central ridge retraining: both variants match the centralized ridge solution up to
floating-point error on all four benchmarks, as expected from the exactness guarantees.
The FedAvg-based methods (FedAvg retrain, FATS, Exact-Fun) are included as task-level
federated baselines; however, because they optimize a cross-entropy objective via
iterative gradient-based training rather than the ridge objective, accuracy differences
relative to our variants should not be interpreted as isolating the effect of exactness alone.
In fact, the two variants generate results almost identical to the central ridge retrain,
which further highlights their invariance to the client data distribution.
We empirically measure this invariance to the client data distribution
on the FeMNIST dataset by computing the relative Frobenius deviation
$\|W_\text{fed}-W_c\|_F/\|W_c\|_F$ between the weight matrix $W_c$ from centralized
training and the corresponding weights $W_\text{fed}$ obtained under federated training
using our two variants with different number of clients, using $K \in \{10, 50, 100\}$.
For each choice of $K$, we partitioned the same pool of 100 writer IDs into $K$ clients
by assigning disjoint sets of writer IDs to clients, so that changing $K$ only changes
how many writer IDs each client holds rather than the underlying data.
We observe from Table~\ref{tab:frobenius_invariance_femnist} that both proposed variants
yield very small but non-zero deviations, indicating practical invariance to the number
of clients and their data partitions.
These residual deviations arise from floating-point truncation, most notably in the
aggregation of second-order statistics.
All experiments are implemented in PyTorch, with the DINOv2 backbone computing features
in its default fp32 precision.
To increase numerical accuracy and further reduce truncation error, the server computes
and aggregates second-order statistics in fp64 precision, which yields deviations on the
order of $10^{-9}$ compared to $10^{-3}$--$10^{-4}$ when using fp32 throughout for
second-order statistics.
This effect is especially pronounced for Variant~A, where forming and accumulating Gram
matrices involves summing large-magnitude entries, so fp64 computation substantially
improves numerical stability, which in turn helps the ridge solution remain
well-conditioned and makes the choice of $\gamma$ less sensitive and easier to tune for
good accuracy.

\begin{table}[t]
  \centering
  \caption{Relative Frobenius deviation $\|W_{\text{fed}}-W_c\|_F/\|W_c\|_F$ on
  \textsc{FeMNIST} for different numbers of clients $K$ and numerical precision used
  for server-side computation and aggregation of second-order statistics.}
  \label{tab:frobenius_invariance_femnist}
\begin{tabular}{l@{\hspace{1.1em}}c@{\hspace{1.1em}}c@{\hspace{1.1em}}c@{\hspace{1.1em}}c@{\hspace{1.1em}}c@{\hspace{1.1em}}c}
    \toprule
    & \multicolumn{3}{c}{Variant A} & \multicolumn{3}{c}{Variant B} \\
    \cmidrule(lr){2-4} \cmidrule(lr){5-7}
    & $K{=}10$ & $K{=}50$ & $K{=}100$ & $K{=}10$ & $K{=}50$ & $K{=}100$ \\
    \midrule
    fp32 & 5.55e-3 & 5.52e-3 & 5.05e-3 & 1.27e-4 & 1.26e-4 & 1.21e-4 \\
    fp64 & 1.47e-9 & 1.42e-9 & 1.42e-9 & 6.18e-10 & 1.37e-9 & 1.42e-9 \\
    \bottomrule 
  \end{tabular}
\end{table}

\subsection{Unlearning Performance}
We evaluate unlearning performance under two complementary scenarios: (i) chunked deletions, where we repeatedly remove large fractions ($20\%$ of the total data) of the training set and monitor both utility and latency, and (ii) burst single-point deletions, where we issue 200 isolated deletion requests and measure per-request unlearning time. In both settings, we run with 100 clients and cache features from the frozen backbone and restrict recomputation during unlearning to the linear head, since this cost is shared by all methods we compare.

\subsubsection{Large-Chunk Data Deletions}
\begin{figure}
\begin{tabular}{cc}
\includegraphics[width=0.49\linewidth]{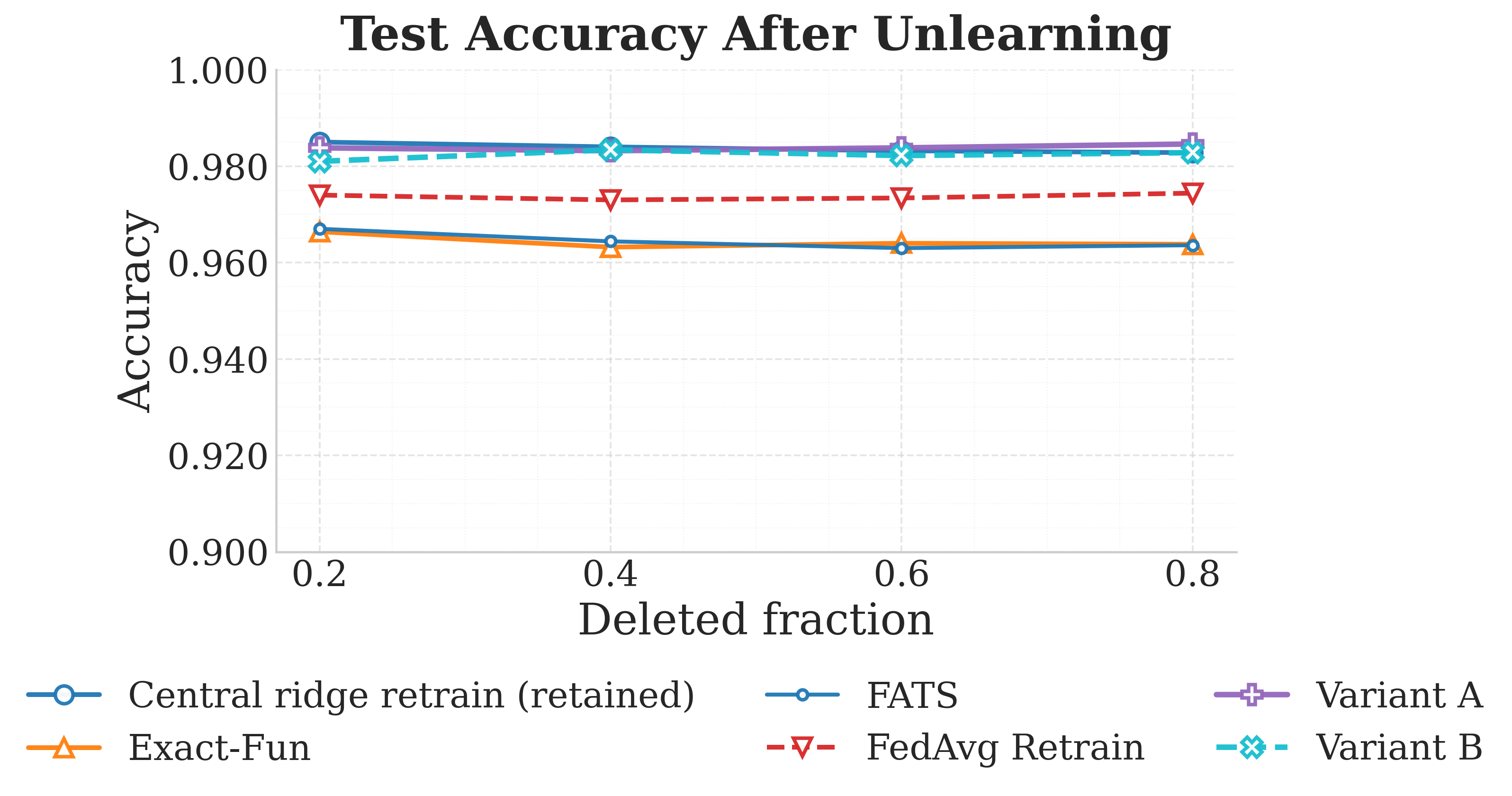} &   \includegraphics[width=0.49\linewidth]{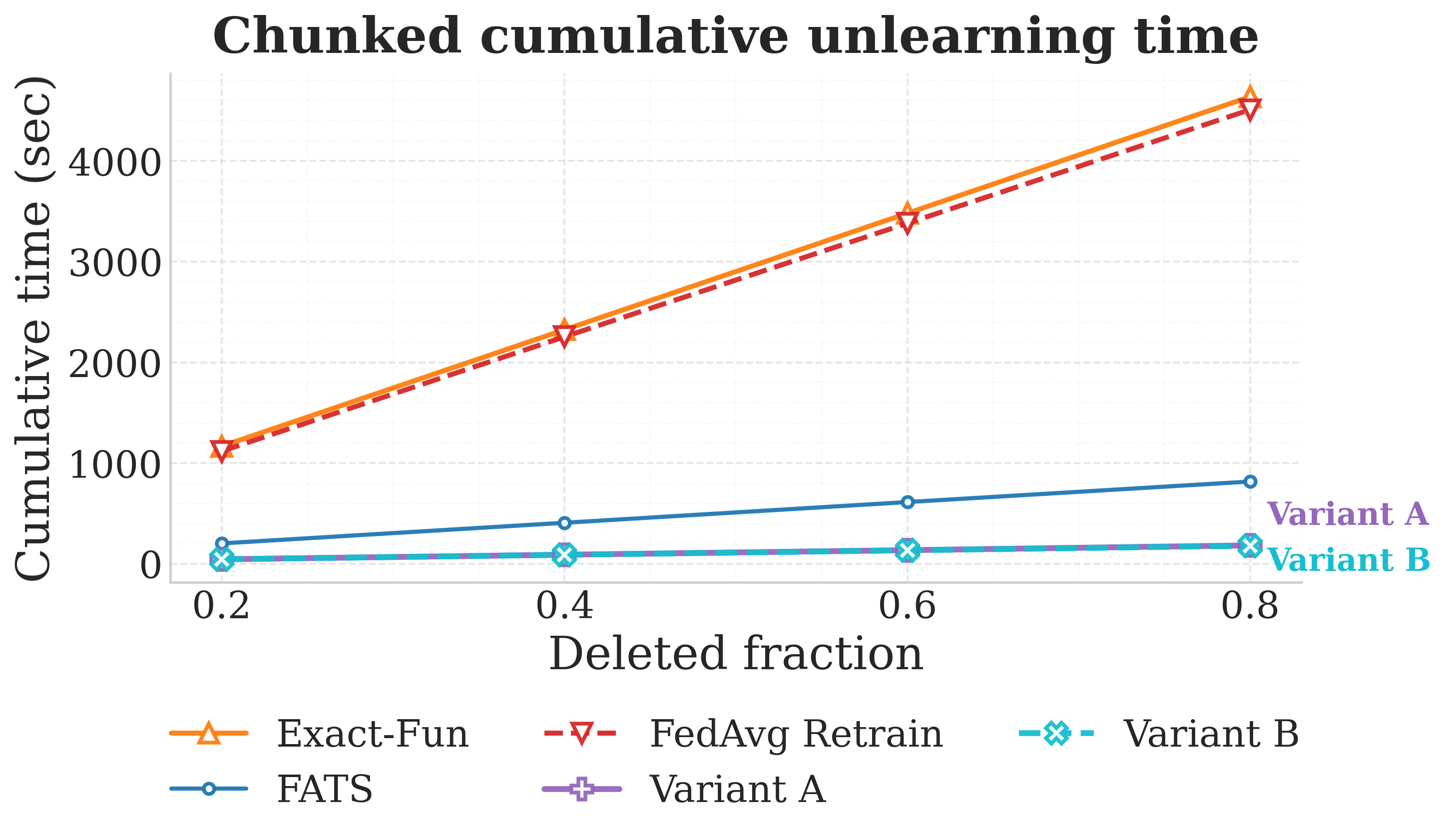} \\
    Test Accuracy & Cumulative time 
\end{tabular}
\caption{ Test Accuracy and Cumulative Time for CIFAR-10 Chunked Deletions. Unlearning performance on CIFAR-10 with repeated deletions of $20\%$ of the total data per step. Deleted samples are randomly selected from all clients. Left: test accuracy tracking utility across deletion rounds. Right: cumulative unlearning time demonstrating system efficiency.}
    \label{fig:chunk_deletion}
\end{figure}
For the chunked deletions in Figure \ref{fig:chunk_deletion}, both proposed variants remain close to the central ridge-retrain baseline across all rounds, exactly preserving the retained-data optimum and consistently outperforming the FedAvg-based baselines. In this large-chunk setting, the skip-retrain heuristics of Exact-Fun and FATS almost
never activate: large deletions trigger the quantized-parameter check
in Exact-Fun and make it likely that deleted samples were seen in the early training
rounds of FATS, forcing near-full retraining, further underscoring the robustness and efficiency of our two exact ridge-head variants in this regime.

The cumulative wall-clock time grows almost linearly with deletion rounds, even though the retained dataset shrinks and retraining would typically become cheaper, because orchestration overhead from synchronizing sampled clients and aggregating their messages dominates while client and server computation is comparatively minor. For Variant A, client-side computation of second-order statistics takes about 6.04s versus roughly 45s total unlearning time per chunk-deletion step, with most latency spent coordinating many clients in each large-chunk deletion round. Since each variant completes an unlearning step in a single communication round, this orchestration-dominated regime naturally favors them over multi-round FedAvg-style retraining, whose cumulative wall-clock time grows more quickly as deletions accumulate.

\subsubsection{Burst Deletions and Add-Back Continual Unlearning/Learning}
\begin{figure}
    \centering
     \includegraphics[width=.7\linewidth]{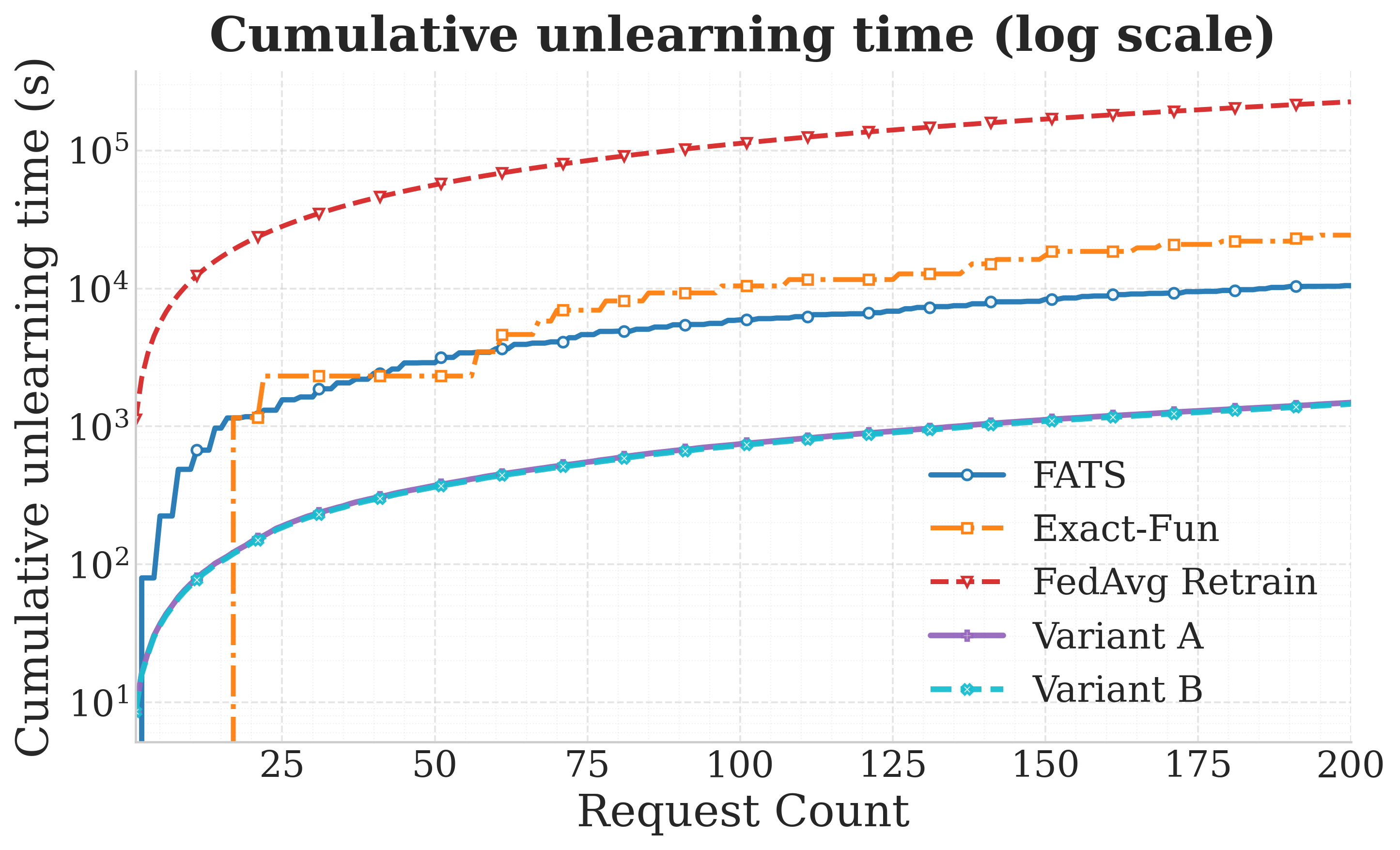}
    \caption{Wall time measurement of 200 unlearning requests (single data points) performed by different methods }
    \label{fig:burst}
\end{figure}

Figure \ref{fig:burst} reports the wall-clock time of all methods on a test of 200 consecutive single-point unlearning requests. This sparse request stream is particularly challenging for many unlearning methods when deletions arrive frequently \cite{tang2025acu} and directly probes how well each approach handles a continual deletion-only unlearning scenario. In this regime, the retraining-skipping mechanisms of Exact-Fun and FATS are effective, so a large fraction of single-point unlearning requests incur essentially zero additional latency. However, whenever retraining is triggered for FATS or Exact-Fun, the resulting partial retraining remains costly and requires the system to maintain multiple checkpoints in order to resume these retraining phases. By contrast, our two proposed variants, which do not require saving multiple checkpoints, incur only a small but consistent cost per deletion: each request involves only a single client–server communication round, and the required second-order statistics can be formed rapidly, yielding a low per-step overhead. As a result, the overall wall-clock cost of our method is orders of magnitude smaller than that of Exact-Fun, FATS, and FedAvg retraining baselines while avoiding any need to maintain model checkpoints. 

To further validate the exactness of the proposed variants over this sequence of 200 deletions and their support for continual learning, we additionally perform a reverse experiment that re-adds the previously deleted points to recover the same ridge solution, and we report the Frobenius deviation $\|W_\text{fed}-W_c\|_F/\|W_c\|_F$ from the central ridge solution after 100 and 200 steps. The resulting small deviations shown in Table \ref{tab:con_deviation} confirm both the numerical stability of the proposed method and its exactness under continual learning and unlearning process.

\begin{table}[t]
\centering
\caption{Relative Frobenius deviation
$\|W_{\text{fed}} - W_c\|_F / \|W_c\|_F$
from the central ridge solution after 100 and 200 single-point deletions and the corresponding add-back experiment on CIFAR-10.}
\label{tab:con_deviation}
\begin{tabular}{@{}lcccc@{}}
\toprule
& \multicolumn{2}{c}{Deletions only} & \multicolumn{2}{c}{Delete + add-back} \\
\cmidrule(lr){2-3} \cmidrule(lr){4-5}
Method & 100 steps & 200 steps & 100 steps & 200 steps \\
\midrule
Variant A & 2.72e-11 & 3.18e-11 & 3.14e-11 & 3.81e-11 \\
Variant B & 3.00e-11 & 3.66e-11 & 2.82e-11 & 5.10e-12 \\
\bottomrule 
\end{tabular}
\end{table}

\section{Discussion, limitations, and ethics}

We discuss the scope of the unlearning guarantee, privacy properties of the transmitted statistics, and numerical considerations for practical deployment.

\paragraph{Scope of unlearning.}
The protocol exactly removes the influence of deleted samples \emph{from the ridge head parameters}, given that the backbone features are fixed. If the backbone itself was trained on revocable data, removing head influence alone is not sufficient; backbone unlearning is an orthogonal problem and remains challenging in general.

\paragraph{Privacy considerations.}
The sufficient statistics $(S,G)$ are aggregate second-order summaries and do not expose individual raw samples, consistent with standard federated learning practice. However, second-order statistics can still leak information about the underlying feature distribution and label frequencies, particularly when $d$ is large and the backbone is shared and fixed across clients. Practical deployments can mitigate this by combining the protocol with secure aggregation and, where required, differential privacy mechanisms applied to the transmitted statistics.

\paragraph{Numerical stability.}
Variant~A (SPD solve) is numerically robust, especially when server-side Gram matrices and solves are computed in fp64 while the backbone runs in fp32. Variant~B (SMW downdates) is exact in exact arithmetic but can accumulate floating-point drift over long downdate sequences; computing SMW operations in fp64 substantially reduces this drift, and periodic resets to Variant~A provide a practical safeguard when accumulated updates become ill-conditioned.


\section{Conclusion}

We presented an exact federated continual unlearning protocol for a ridge head trained on frozen foundation-model features. By treating the ridge sufficient statistics as a retained-set ledger, the server supports arbitrary interleavings of add and delete requests and maintains a head that is, in exact arithmetic, pointwise identical to centralized retraining on the retained data. We provided deterministic retrain-equivalence guarantees, order and partition invariance, two practical server variants, and a Bayesian posterior certificate implying zero KL divergence between the unlearned and retrained posteriors.

This establishes a strong exact unlearning baseline for the increasingly common frozen backbone and lightweight head deployment regime. Experimentally, both variants match centralized ridge retraining in utility across all benchmarks; the FedAvg-based baselines are also competitive, though a direct accuracy comparison is confounded by their use of a different training objective. Exactness is preserved across different client partitions and throughout continual add--delete streams, and the dominant practical bottleneck is round orchestration rather than computation, making the single-round protocol particularly well suited to realistic federated deployments.


\section*{Acknowledgments}
This research was financially supported by Engineering and Physical Sciences Research Council (EPSRC), United Kingdom, [SUSTAIN Manufacturing Hub EP/S018107/1].\\
GM acknowledges support from a UKRI AI Turing Acceleration Fellowship (EPSRC EP/V024868/1).

\bibliographystyle{splncs04}
\bibliography{references}

\fi
\ifmainonly
\else

\newpage

\appendix
\section*{Appendix}
\setcounter{subsection}{0}
\renewcommand\thesubsection{A\arabic{subsection}}

\subsection{Experimental Hyperparameter and Sampling}
\label{app:exp_params}
For Exact-Fun, we use uniform parameter quantization with step size $\alpha = 0.08$, i.e., $q(W)=\text{round}(W/\alpha)\alpha$, applied only to the linear heads. On the dataset we test, smaller step sizes such as $\alpha=0.06$ made the linear head overly sensitive to deletions, so even single-point removals frequently triggered partial retraining, sacrificing precision without improving unlearning efficiency; we therefore choose $\alpha= 0.08$ to balance unlearning efficiency and performance.

The round of FedAvg simulations and client and data sampling rules are reported in Table \ref{tab:params} below.
\begin{table}[h]
  \centering
  \caption{The hyperparameter setting for baselines.}
  \label{tab:params}
  \begin{tabular}{lcccc}
  \toprule
  Method & Dataset & \#Rounds &
  \makecell{Clients\\per round} &
  \makecell{Local data\\per round} \\
  \midrule
  FedAvg retrain & CIFAR-10/100 & 120 & 20 & full local data \\
  Exact-Fun      & CIFAR-10/100 & 120 & 20 & full local data \\
  FATS           & CIFAR-10/100 & 60  & 5  & 64 samples\\
  FedAvg retrain & FeMNIST      & 120 & 20 & full local data \\
  Exact-Fun      & FeMNIST      & 120 & 20 & full local data \\
  FATS           & FeMNIST      & 120 & 5  & 50 samples\\
  FedAvg retrain & Sentiment140 & 100 & 20 & full local data \\
  Exact-Fun      & Sentiment140 & 100 & 20 & full local data \\
  FATS           & Sentiment140 & 100 & 5  & 300 samples \\
  \bottomrule
\end{tabular}
\end{table}

\subsection{Communication-efficient approximations and bounds}
Modern foundation models such as DINO and pretrained LLMs typically expose moderate feature dimensions (e.g., 768 or 1024), which our lightweight linear-head methods can handle directly. At the same time, they are also compatible with much higher-dimensional features via approximate updates. Transmitting full $\Delta S\in\R^{d\times d}$ can be expensive when $d$ is large (e.g., high-dimensional foundation features), so we introduce approximate Gram updates together with bounds that quantify the drift from exact retrain equivalence.

\subsubsection{Low-rank/sketched Gram updates}
Suppose a client approximates $\Delta S$ by a rank-$r$ matrix $\Delta S_r$ (e.g., truncated eigendecomposition, random projections,
or sketching) and transmits $\Delta S_r$ (or its factor) along with $\Delta G$.

Let $H=S+\gamma I$.
For an add update, define
\[
T_{\mathrm{ex}}=(H+\Delta S)^{-1},\qquad T_{\mathrm{ap}}=(H+\Delta S_r)^{-1},\qquad E=\Delta S-\Delta S_r.
\]
A standard resolvent argument yields the following perturbation bound.

\begin{theorem}[Add-update perturbation bound]
\label{thm:add_bound}
Assume $\normtwo{T_{\mathrm{ap}}E}<1$. Then
\[
\normtwo{T_{\mathrm{ex}}-T_{\mathrm{ap}}}
\le
\frac{\normtwo{T_{\mathrm{ap}}}^2\,\normtwo{E}}{1-\normtwo{T_{\mathrm{ap}}E}}.
\]
If $W_{\mathrm{ex}}=T_{\mathrm{ex}}(G+\Delta G)$ and $W_{\mathrm{ap}}=T_{\mathrm{ap}}(G+\Delta G)$, then
\[
\normtwo{W_{\mathrm{ex}}-W_{\mathrm{ap}}}
\le
\normtwo{T_{\mathrm{ex}}-T_{\mathrm{ap}}}\,\normtwo{G+\Delta G}.
\]
\end{theorem}

\begin{proof}
Write $T_{\mathrm{ex}}=(H+\Delta S_r+E)^{-1} = (I+T_{\mathrm{ap}}E)^{-1}T_{\mathrm{ap}}$ and expand
$T_{\mathrm{ex}}-T_{\mathrm{ap}}= - (I+T_{\mathrm{ap}}E)^{-1}T_{\mathrm{ap}}E T_{\mathrm{ap}}$.
Bound using $\normtwo{(I+X)^{-1}}\le 1/(1-\normtwo{X})$ when $\normtwo{X}<1$.
\end{proof}

Delete-update bounds are analogous when $H-\Delta S_r\succ 0$ and the perturbation is feasible.

\subsubsection{Periodic exact reset}
A practical design is to run many approximate steps and periodically \emph{reset} by recomputing $W_t$ exactly via Variant~A from the aggregated ledger $(S_t,G_t)$.
Theorems like \ref{thm:add_bound} quantify drift between resets in terms of the neglected spectral mass $\normtwo{E}$.

\fi

\end{document}